\documentclass[letterpaper, 10pt, conference]{IEEEconf}  
\IEEEoverridecommandlockouts 
\usepackage[utf8]{inputenc}

\usepackage{amsmath}
\usepackage{xcolor}
\usepackage{xspace}
\usepackage{siunitx}
\usepackage{amssymb}
\usepackage{mathtools} 
\usepackage{multirow}
\usepackage{wrapfig}
\usepackage[leftmargin=0.70em]{quoting}
\usepackage{algorithm}
\usepackage{algpseudocode}
\usepackage{tikz}

\usepackage{subfigure}

\newtheorem{definition}{\bf Definition}
\newtheorem{property}{\bf Property}
\newtheorem{lemma}{\bf Lemma}

\newcommand{\gobble}[1]{}

\makeatletter
\newcommand{\dbloverline}[1]{\overline{\dbl@overline{#1}}}
\newcommand{\dbl@overline}[1]{\mathpalette\dbl@@overline{#1}}
\newcommand{\dbl@@overline}[2]{%
  \begingroup
  \sbox\z@{$\m@th#1\overline{#2}$}%
  \ht\z@=\dimexpr\ht\z@-2\dbl@adjust{#1}\relax
  \box\z@
  \ifx#1\scriptstyle\kern-\scriptspace\else
  \ifx#1\scriptscriptstyle\kern-\scriptspace\fi\fi
  \endgroup
}
\newcommand{\dbl@adjust}[1]{%
  \fontdimen8
  \ifx#1\displaystyle\textfont\else
  \ifx#1\textstyle\textfont\else
  \ifx#1\scriptstyle\scriptfont\else
  \scriptscriptfont\fi\fi\fi 3
}
\let\NAT@parse\undefined 
\makeatother
\usepackage[hidelinks]{hyperref}

\title{\LARGE \bf 
A Model for Optimal Resilient Planning Subject to Fallible Actuators
}

\author{Kyle Baldes, Diptanil Chaudhuri, Jason M. O'Kane, and Dylan A. Shell
\thanks{\hspace*{-2.2ex}\scriptsize The authors are all affiliated with Dept. of Computer Science \& Engineering, 
Texas A\&M University, College Station, TX, USA.
        {\tt\scriptsize \{baldesk$\,|\,$dshell\}@tamu.edu}.
        }
}
\protected\def\verythinspace{%
  \ifmmode
    \mskip0.5\thinmuskip
  \else
    \ifhmode
      \kern0.08334em
    \fi
  \fi
}

\providecommand{\Pr}{\ensuremath{\mathrm{Pr}}}
\providecommand{\Reals}{\ensuremath{\mathbb{R}}}
\providecommand{\Nats}{\ensuremath{\mathbb{N}_{>0}}}

\providecommand{\opt}{\ensuremath{\star}}

\providecommand{\mdp}{MDP\xspace}

\providecommand{\mdps}{MDPs\xspace}

\providecommand{\famdp}{FA-MDP\xspace}

\providecommand{\famdps}{FA-MDPs\xspace}
\providecommand{\monomdp}{monolithic \mdp}

\providecommand{\fset}[1]{\ensuremath{[{#1}]}}
\providecommand{\pow}[1]{\ensuremath{{2}^{#1}}\xspace}

\providecommand{\coll}[1]{\ensuremath{\boldsymbol{#1}}\xspace}
\providecommand{\mon}{\ensuremath{_{\rm mon}}}
\providecommand{\all}[1]{\ensuremath{{#1}^{\!\verythinspace\flat}}}

\providecommand{\babybell}{\ensuremath{{\rm\bf b}}}
\providecommand{\typemdp}[1]{\scalebox{0.85}{\ensuremath{\mathcal{#1}}}}
\providecommand{\node}[2]{\ensuremath{\langle{#1};{#2}\rangle}}
\providecommand{\typegr}[1]{\ensuremath{\mathcal{#1}}}
\providecommand{\idx}[1]{{\scalebox{0.8}{\ensuremath{({#1})}}}}
\providecommand{\maxnorm}[1]{\ensuremath{\left\|{#1}\right\|}}
\providecommand{\smallmaxnorm}[1]{\ensuremath{\big\|{#1}\big\|}}
\providecommand{\abs}[1]{\ensuremath{\left|{#1}\right|}}

\providecommand{\lb}{\ensuremath{{(\gamma\cdot\relM)}}}
\providecommand{\tgt}{\ensuremath{{\rm desired}}}
\providecommand{\goal}{\ensuremath{{\rm goal}}}

\providecommand{\rel}{\ensuremath{\rho}}
\providecommand{\relm}[1]{\ensuremath{\hat{\rel}_{#1}}}
\providecommand{\relmin}[1]{\ensuremath{\check{\rel}_{#1}}}
\providecommand{\relM}{\ensuremath{\overline{\rel}}}

\providecommand{\panglossian}{panglossian\xspace}

\providecommand{\gridworld}{gridworld\xspace}
\providecommand{\hotstarting}{hot-starting\xspace}

\providecommand{\monoplanner}{Monolithic Planner\xspace}
\providecommand{\latticeplanner}{Lattice Planner\xspace}
\providecommand{\latticeplannerhotstart}{Hot-Start Lattice Planner\xspace}

\providecommand{\reconfigurable}{reconfigurable\xspace}
\providecommand{\Faulttolerance}{Fault tolerance\xspace}

\providecommand{\faultdetection}{fault detection\xspace}

\providecommand{\locA}{\textcolor{red!50!orange}{\textsf{A}}\xspace}
\providecommand{\locB}{\textcolor{red!50!orange}{\textsf{B}}\xspace}
\providecommand{\locC}{\textcolor{red!50!orange}{\textsf{C}}\xspace}
\providecommand{\locD}{\textcolor{red!50!orange}{\textsf{D}}\xspace}

\let\emptyset\varnothing 

\begin{document}
\maketitle

\thispagestyle{plain}
\pagestyle{plain}


\begin{abstract}
Robots incurring component failures ought to adapt their behavior to best realize still-attainable goals under reduced capacity.
We formulate the problem of planning with actuators known \emph{a priori} to be susceptible to failure within the Markov Decision Processes (MDP) framework.
The model captures utilization-driven malfunction and state-action dependent likelihoods of actuator failure in order to enable reasoning about potential impairment and the long-term implications of impoverished future control.
This leads to behavior differing qualitatively from plans which ignore failure.
As actuators malfunction, there are combinatorially many configurations which can arise.
We identify opportunities to save computation through re-use, exploiting the observation that differing configurations yield closely related problems.
Our results show how strategic solutions are obtained so robots can respond when failures do occur---for instance, in prudently scheduling utilization in order to keep critical actuators in reserve.
\end{abstract}

\section{Introduction: Motivation and related work}
Robot hardware malfunctions and faults aren't rare events.
One means of mitigation to engineer rugged components;
another is redundancy, as in \reconfigurable modular robots~\cite{ahmadzadeh2015modular}.
But over-provisioning, the idea underlying both approaches, drives up weight, energy, and economic costs.  
An alternative is to imbue systems with the ability to anticipate and tolerate their own deterioration. 
The present paper adopts this latter philosophy. 
It focuses on the specific problem of actuator failure---a known issue already identified as important in prior work (e.g., see \cite{lu2023joint,xiao2020robust,yoo2012actuator,hsiao2012hierarchical}).
We consider a robot that plans and then executes a policy, reasoning not only about traditional uncertainty in the form of imprecise control (i.e., stochasticity in action outcomes), but also the possibility that each use of an actuator may cause it to fail. 

We treat this problem through the Markov Decision Processes (\mdp) framework\,\cite{lavalle06planning, russell09artificial, sutton18reinforcement, bertsekas19reinforcement}.
It is standard practice within robotics to solve \mdps under the expected total discounted reward criterion and,
although the discount factor is often explained as quantifying the preference for immediate versus future rewards, Puterman points out:
\begin{quoting}
\small\par\textsl{
``Discounting may be viewed in another way. The decision maker values policies
according to the expected total reward criterion; however, the horizon length\dots 
is random and independent of the actions of the decision maker. Such randomness in
the horizon length might be due to \dots death\dots'' \cite[\S 5.3, .p.~125]{puterman94mdp}}
\end{quoting}
Under this interpretation current approaches treat robot systems as failing in a catastrophic (or all-or-nothing) way.
The approach we present, thus, can be considered a generalization of the standard model where, rather than a failure debilitating the entire device, it retains some abilities to operate albeit in a diminished fashion.

A direct way to treat failure of an actuator, once the system has determined that the associated actions are no longer feasible, is to re-plan using those capabilities which remain.
This approach is \emph{\panglossian}: despite failures having occurred, the robot plans as if there will be no future failures. 

Fundamentally, re-planning adopts a reactive point of view\,---\,it lacks foresight. 
Instead, one might reason strategically about when and where to employ actions,
taking into account the implications for the future, considering
the structure of the state-space and, critically, the
\emph{reliability} of the available actuators.  Because this involves reasoning
about (probabilistically weighted) impaired future operation,
generating such behavior is anticipatory rather than reactive.

\begin{figure}
\hspace*{10pt}
\begin{minipage}{1.0\linewidth}
\centering
    \includegraphics[width=1.8cm, trim=-9pt 130pt 0pt 0pt, clip]{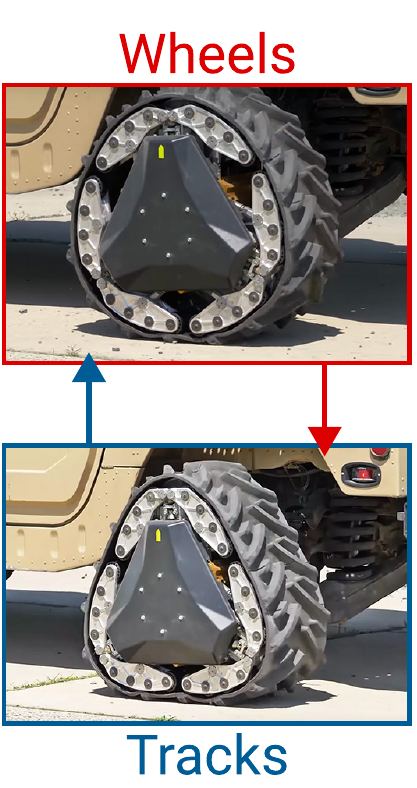}\hspace*{-4pt}
    \includegraphics[width=6.4cm, trim=25pt 190pt 10pt 40pt]{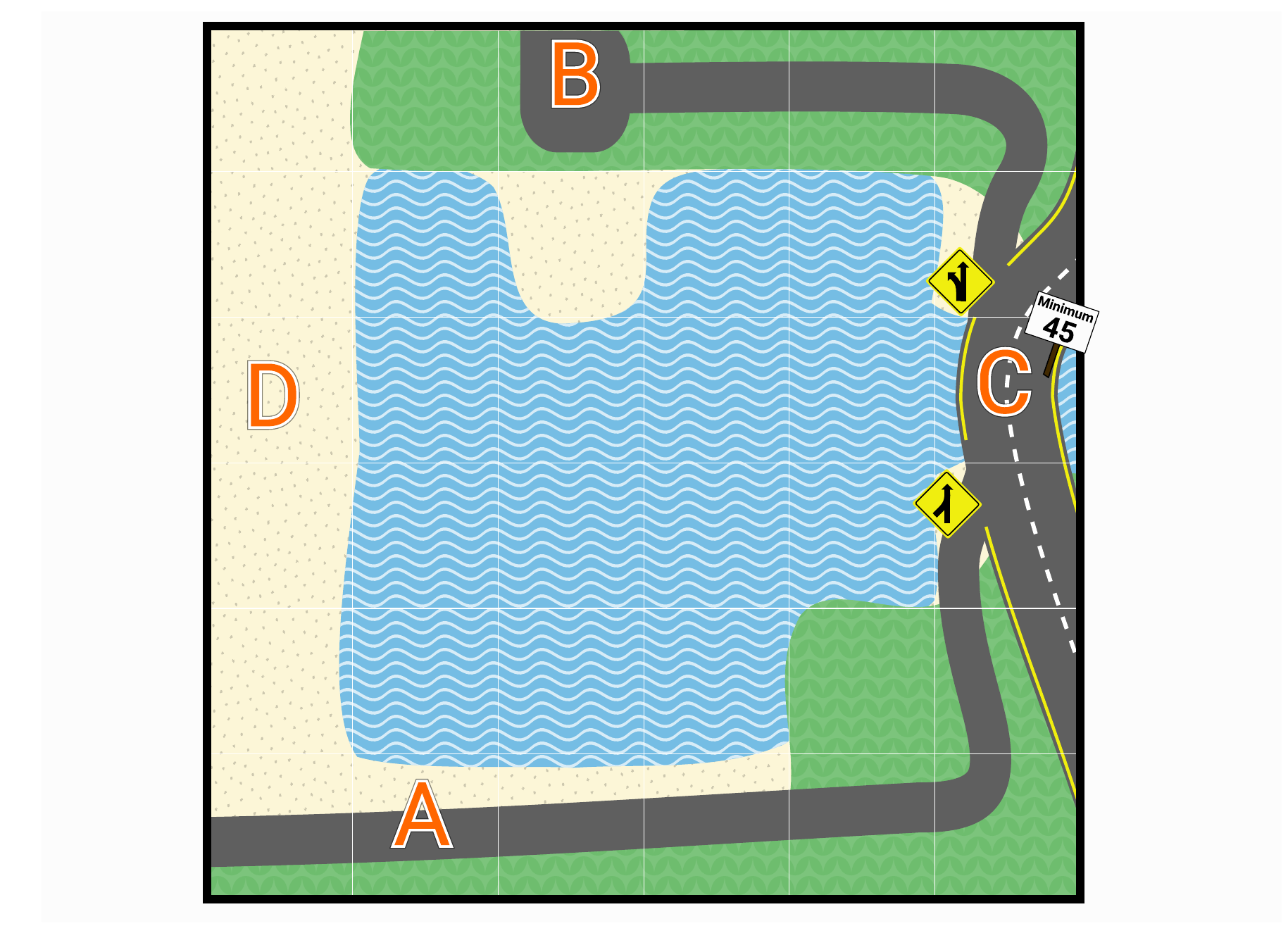}\vspace*{5pt}
    \hspace*{-0.9cm}\setlength{\fboxsep}{0pt}\fbox{\includegraphics[height=1.80cm, trim=-1pt 0pt 0pt 384.5pt, clip]{figures/humvee4.pdf}}\hspace*{115pt}\phantom{x}
    \vspace*{-1pt}
\end{minipage}
\caption{A robot equipped with \emph{wheels} and \emph{tracks} travels from location \locA to \locB subject to both motion uncertainty and the possibility of actuator failure.\label{fig:motivation}}
\vspace*{-12pt}
\end{figure}

\newcommand{\policyheight}{4.55cm}
\begin{figure*}[t!]
    \centering
    \begin{minipage}{\linewidth}
    \centering
    \hfill
        \subfigure[A panglossian policy.]{
        \includegraphics[height=\policyheight]{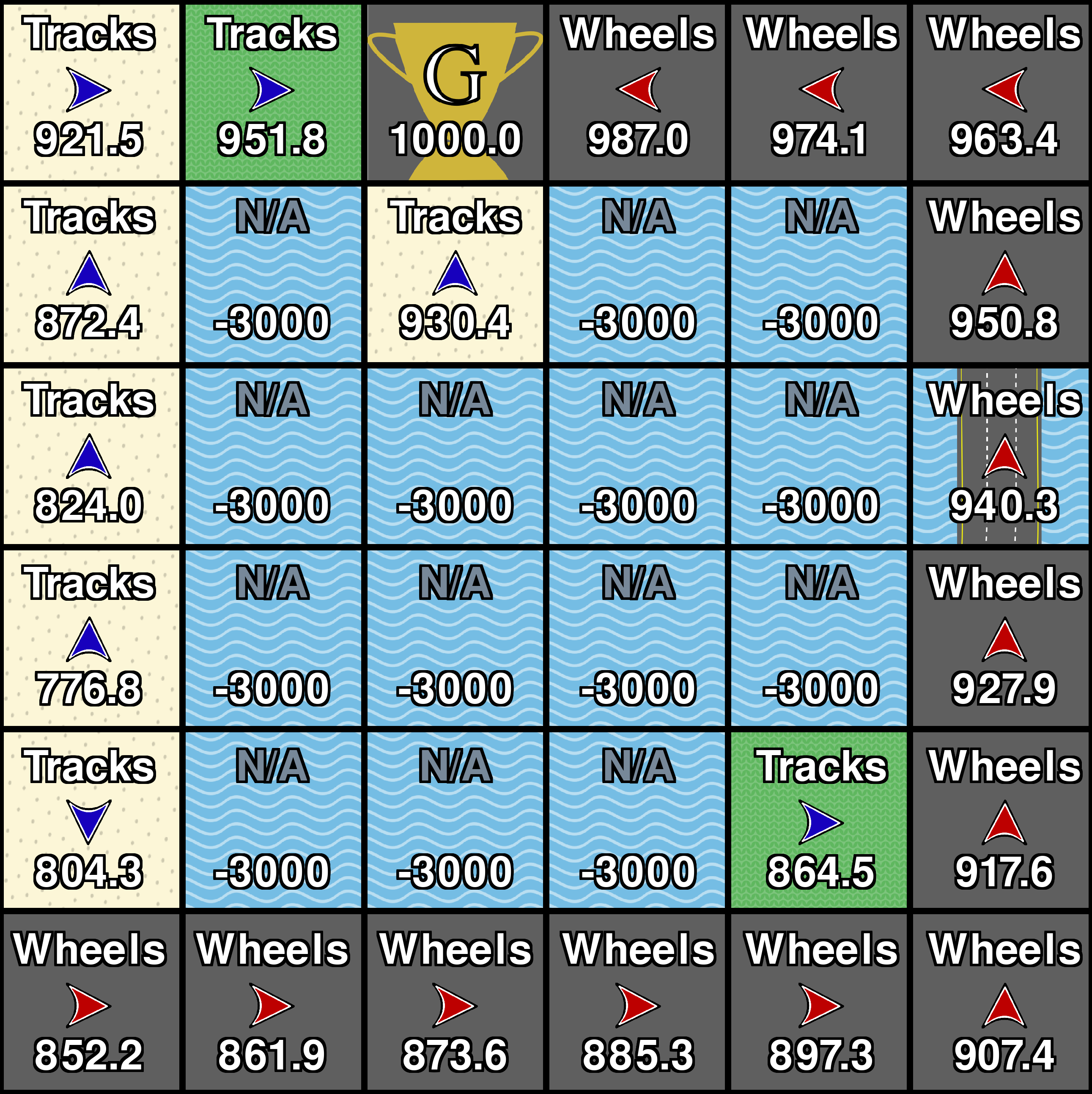}
        \label{fig:panglossianpolicy}
    }
    \hfill
    \subfigure[Failure aware policy.]{
        \includegraphics[height=\policyheight]{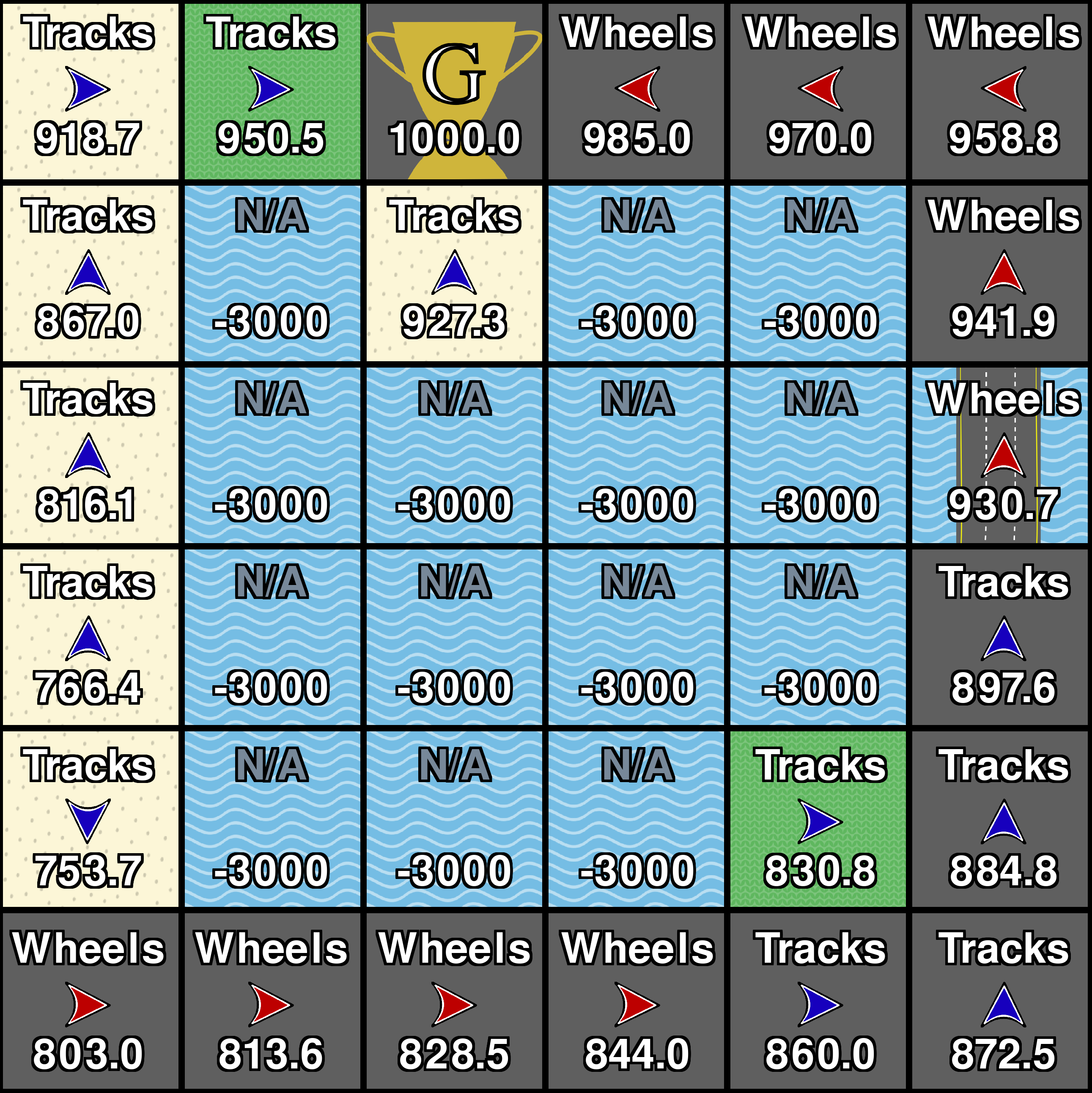}
        \label{fig:famdpPlannerTop}
    }
    \hfill
    \subfigure[Panglossian policy execution.]{
        \includegraphics[height=\policyheight, trim=80pt 15pt 70pt 15pt, clip]{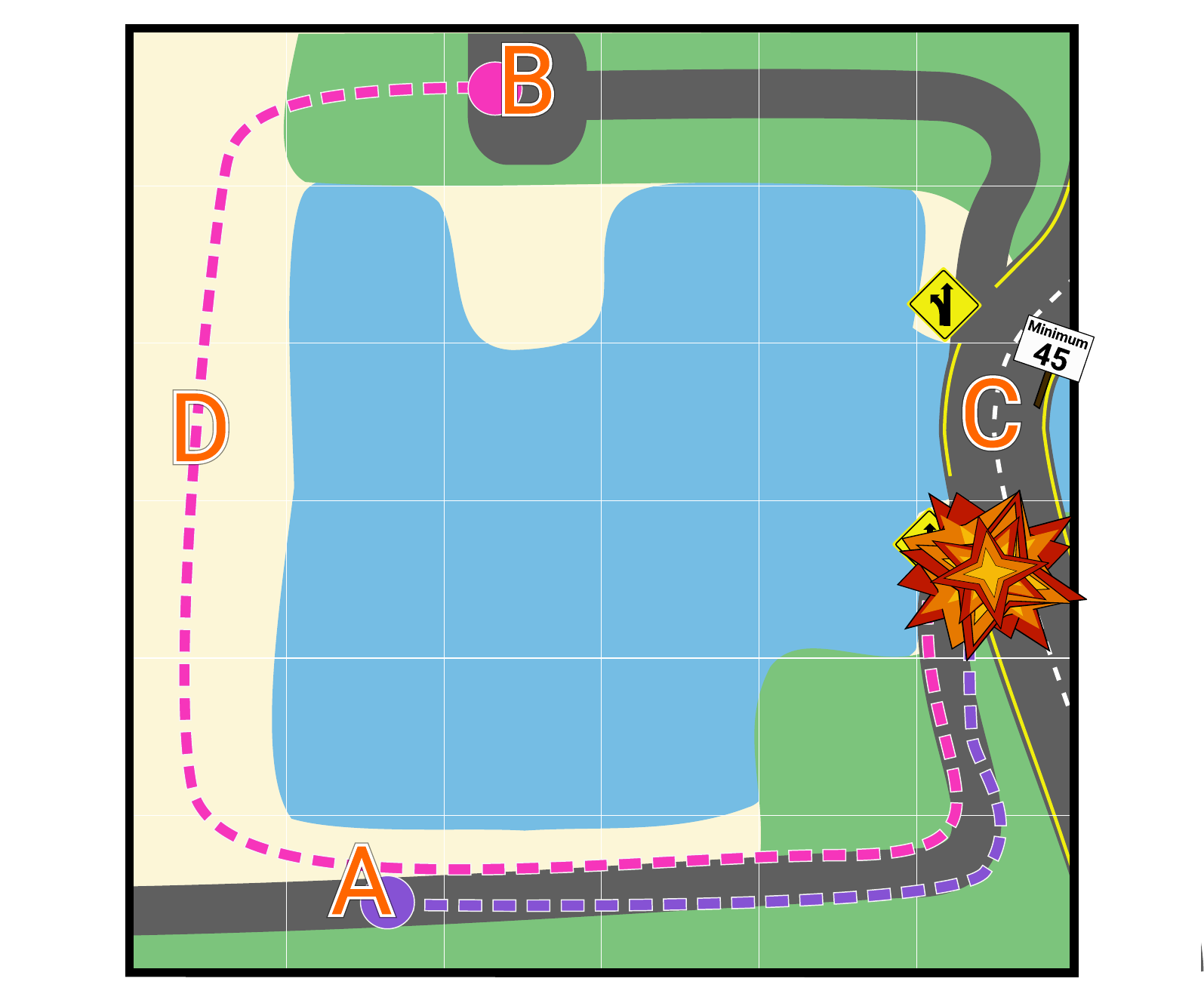}
        \label{fig:panglossianfailure}
    }
    \hfill\phantom{.}
    \vspace*{-6pt}
    \end{minipage}
    \begin{minipage}{1.0\textwidth}
        \caption{Computed policies for the situation in \autoref{fig:motivation} for specific reward values.
        \textsl{(a)} Policy derived from a $6\times6$ \gridworld representation of \autoref{fig:motivation} not accounting for actuator failures. 
        \textsl{(b)}  Failure aware policy.
        \textsl{(c)} Execution of the \panglossian policy in (a) that, after failure, results in undesired behavior.
        At each grid cell, the control with the maximum expected future reward is displayed by indicating the actuator, direction of travel, and the cost-to-go for that control. At the goal state, denoted by G, all controls have the same expected future reward.}
    \end{minipage}
    \label{fig:policycomparison}
    \vspace*{-10pt}
\end{figure*}
\subsection{Motivating scenario}

Consider the scenario in \autoref{fig:motivation}: a robot must navigate from~\locA to \locB. 
It has wheels and tracks, through a \reconfigurable wheel-track system like that of~\cite{lyness2019rwt} that allows the robot to rapidly transition between them. 
Only one actuator may be used at a time so the robot must decide when to use each.  
The tracks and wheels have complementary capabilities: 
wheels enable the robot to move faster on the road at the price of reliability; 
tracks, in turn, trade speed on the road for reliability at moderate speed anywhere. 
The tracks' slower speed makes them unusable along sections of road with a minimum speed, such as the one on the bridge designated point~\locC,
where a threshold is  imposed to alleviate highway congestion.
This intricate actuation arrangement is not infallible, as there is a nonzero probability of failure for both wheels and tracks. 
The rate of failure is depends on usage, and the occurrence of malfunction is detected when employing the actuator. 
The system may recover from a single failure by switching when a failure is detected, however a subsequent failure will leave the robot immobile.

To travel from \locA to \locB, the robot constructs a policy informing it about the controls it will use at each location.
Since the robot is unable to operate in the water, there are two high-level classes of path along which it can 
travel to get from \locA to \locB. 
The first, called route~1, takes the robot along the access road, and briefly along the highway to cross the bridge at point~\locC, before taking the top access road to \locB.
The second route requires the robot to travel off-road through the sand at point~\locD, and then across the grass before reaching point~\locB. 
A robot which disregards the potential for actuator failure produces a policy (which we designate the \panglossian policy) seen in \autoref{fig:panglossianpolicy}. 
This policy has the robot traveling along route~1, utilizing the wheels along the entire journey as their higher speed garners a higher expected reward than using the tracks.

Now consider a policy which does account for actuator failure, depicted in \autoref{fig:famdpPlannerTop}.
This policy also chooses to travel along route~1, but instead elects to use both actuators along the way.
Most significantly, the wheels are held in reserve before reaching the critical point \locC, where they are required to cross the bridge along the highway.
Once close enough to the goal, the policy favors the wheels for their speed along the access road---as even if they fail, the robot can then still succeed by falling back to the tracks.

The implications of the policy difference is illustrated in \autoref{fig:panglossianfailure}. 
Here a myopic robot executing the \panglossian policy travels along route~1 using the wheels (highlighted in \textcolor{violet}{violet}).
However, the robot is alerted to the failure of the wheels as it attempts to use them just before point~\locC.
Now, unable to cross the bridge, the robot must turn-around and travel back to location~\locA before continuing on route 2 to location~\locB (in \textcolor{magenta}{magenta}). 
A robot with policy which anticipates failures will not encounter such a scenario as the wheels are preserved, ensuring their availability at the critical juncture.
Moreover, the difference in outcomes can be arbitrarily bad,
for example, when a robot crosses a bridge onto an island 
and,  owing to a failure, the stricken robot is now stranded.

This trivial case exemplifies the necessity to equip robots with information that characterizes how their capabilities might change with use and time. 
Metrics such as component failure rates are already widely adopted, suggesting that a reliability metric based on utilization is reasonable.

\subsection{Related work}
The distinction between an over-provisioned system and one
which adjusts its performance in light of operating challenges has been
proposed as the difference between resilient and
robust behavior~\cite{prorok2021beyond, zhang2017resilient}, both having
seen much recent attention.

\Faulttolerance and \faultdetection have a long history in
robotics~\cite{srinivas77thesis,donald88geometric}; the early
work is summarized in the survey
of~\cite{visinsky94robotic}, and was revisited more recently in~\cite{walker2022robot}. 
Work in this vein generally involves introspection or a process of scanning in order to detect faults and diagnosis (i.e. attributing and tracking faultiness~\cite{verma2004real}, or determining the bounds of one's own performance~\cite{gautam22}).  
The literature on diagnosis and \faultdetection is considerable, 
for a survey specifically with a focus on robots see \cite{khalastchi2018fault}.
Some research also emphasizes specific recovery mechanisms, e.g.,~\cite{crestani15enhancing}.

The control theory community has devoted considerable attention to the topic of
fault tolerant control, with wide application in
safety critical systems (e.g., aircraft and nuclear facilities~\cite{kumar2018reliability}).
The recent survey of~\cite{amin2019}
evaluates different methodologies for detecting,
isolating and handling faults in both sensors and actuators.
Those methods handle low-level \faultdetection and tolerance,
but neglect to incorporate actuator reliability into strategic decisions 
that can affect how high-level goals are attained, e.g., picking
vastly different classes of solution.

For multi-robot systems, 
recent work on
pursuit-evasion computes plans robust to a subset of robots 
perishing~\cite{olsen22robust}. 
Also, work has examined robot reliability 
in teams of unmanned aerial vehicles (UAVs)~\cite{ure2013uav}, and
during multi-robot task
allocation~\cite{stancliff2009taskallocation}.
The latter provides the aphorism ``in order to successfully plan we must `plan
to fail'\thinspace''---\,a statement which succinctly epitomize the methodology
behind our work wherein we 
incorporate knowledge of component reliability of a single robot into decision making.

\section{Notation, preliminaries, and assumptions}

\subsection{Notation}
Let $\fset{n}$ be the set $\{1, 2,\dots,n\}$,
and $\pow{X}$ denote the powerset of $X$.
We flatten a
collection of sets $\coll{V}$ as $\all{\coll{V}} = \cup_{V_i \in \coll{V}} V_i$.

\subsection{Preliminaries}

Recall the following standard definition:
\begin{definition}[\cite{puterman94mdp}]
A \emph{Markov Decision Process (\mdp)}

is a 5-tuple $\langle S, U, T, R, \gamma \rangle$ where
{
\begin{itemize}
\item $S = \{s_1, s_2, \dots, s_{|S|}\}$ is the finite set of states; 
\item $U = \{u_1, u_2, \dots, u_{|U|}\}$ is the finite set of controls; 
\item $T: S\times U \times S \to [0,1]$ is the transition dynamics,
describing the stochastic state transitions of the system,
assumed to be Markovian in the states, where $\forall t, \Pr(s^{t+1} = s' | s^{t}
= s, u^{t} = u) = T(s',u,s)$;
\item $R: S\times U\to \Reals$ is the reward,
specifying that value $R(s,u)$ is obtained by executing control $u$ in state~$s$;
\item $\gamma \in [0,1)$ is the discount factor.
\end{itemize}}

\end{definition}
\smallskip

One solves an \mdp by providing a \emph{policy}, \mbox{$\pi: S \to U$,} which
prescribes that control $\pi(s)$ be executed in state $s$.  

For policy $\pi$, starting from some initial state \mbox{$s^0 \in S$},
we will consider the following criterion:
\begin{definition}For \mdp $\typemdp{M}= \langle S, U, T, R, \gamma \rangle$ 
the \emph{expected cumulative reward} of policy $\pi$ is
\begin{equation}
E[\pi \vert \typemdp{M}] =  \mathop{\scalebox{1.35}{$\mathbb{E}$}}\limits_{s^0,s^1,\dots}\left(\sum_{t = 0}^{\infty} (\gamma^t) R\big(s^t, \pi(s^t)\big) \right), \label{eq:exp-reward}
\end{equation}
\noindent where the expectation is over sequences $s^0,s^1,s^2,\dots$ 
arising with probabilities $T(s^{i},\pi(s^{t-1}), s^{t-1})$, for $t = 1, 2, \dots$.
\end{definition}

We seek $\pi^\opt(\cdot)$
maximizing \eqref{eq:exp-reward}, an \emph{optimal policy}.

\subsection{Assumptions}

\newcommand{\qq}[1]{\phantom{} \textbf{#1.}~}

\noindent The model explicitly assumes three things:---\\
\qq{1}Failure can be detected.\\
\qq{2}Failure is irreparable and considered binary. \\
\qq{3}Efficiency cannot increase after failure of an actuator. \\[-8pt]

Failure is detectable in that,
after attempting to execute an action, 
there is some indicator that failure has
occurred when one has. We model complete failure of each component so it remains
inoperative thereafter.  The treatment
does directly allow, as special cases, a variety of models where discovery of
failure is natural: for instance, as when an actuator's failure is learned only
after trying it and then discovering a non-transition which only occurs upon
failure. 

\label{sec:assumptions-pathological}
We specifically rule out the rather pathological cases where after a failure has occurred, the robot has higher expected reward over future actions. This bounds the performance of a particular set of actuators. Importantly, we do not rule out the possibility that in the process of failing the robot transitions to states with higher expected reward. Still the mere reduction in the number of available actuators in a state cannot itself improve the state's expected value.

\section{Failure model}

Consider a decision process where the set of controls is not modeled merely as
a flat set, but that each control is generated by a subsystem\,---which we will
designate as an `actuator'---\,that is responsible for multiple controls.  
These are the atomic units of failure that we model.
Let the robot have $m$ actuators so that the controls are partitioned into
disjoint sets $U_1, U_2, \dots, U_m$. For example, $U_2 = \{u_5, u_{6}, u_{7},
u_{8}\}$ gives the \num{4} controls that actuator \num{2} is involved in
generating. Then we might consider \mdp $M= \langle S, \;U_1 \cup U_2 \cup
\dots \cup U_m\;, T, R, \gamma \rangle.$

Imagine that the robot attempts to execute some control $u_j$ but that as it does
so, with some probability, 
the actuator fails.
The state outcome does not follow the usual $T(\cdot, u_j, \cdot)$ rule, and the actuator
becomes inoperative so that, from this point forward, none of the controls it
generates can be executed.  Per {Assumption 1}, this failure can be
reliably determined either because the state transition is unusual, or by some
other detector (one might imagine smoke, odd sounds, etc.). 
We treat actuator failures as independent statistical events:
while greater use of the actuator is more likely to cause failure,
one actuator failing does not cause another to fail. 
But one slight subtlety: 
because an actuator fails, the controls it generates are no longer
available, this may cause other controls to be executed than would have
been, thereby increasing the chance of another actuator failing.
This correlation, through increased use, \emph{is} captured in the model.

Here is our definition:
\hspace{-5pt}\begin{definition}
A \emph{Fallible Actuator \mdp (\famdp)}
is an 8-tuple 
$\langle S, m, \coll{U}, T, F, R, \rel, \gamma\rangle$ 
where
\begin{itemize}
\item $S = \{s_1, s_2, \dots, s_{|S|}\}$ is the finite set of states; 
\item $m \in \Nats$ is the number of actuators;
\item $\coll{U} = \{U_1, U_2, \dots, U_m\}$ is a collection of disjoint control sets. 
\item $T: S\times \all{\coll{U}} \times S \to [0,1]$ is the regular (or nominal) transition
model, describing the stochastic state transitions of the system,
assumed to be Markovian in the states, where $\forall t, \Pr(s^{t+1} = s' | s^{t}
= s, u^{t} = u) = T(s',u,s)$;
\item $F: S\times \all{\coll{U}} \times S \to [0,1]$ is the malfunctioning 
transition dynamics, describing the stochastic state transitions of the system
as the control fails,
Markovian in the states, where $\forall t, \Pr(s^{t+1}\!=\!s'| s^{t}\!=\!s, u^{t}\!=\!u) =\allowbreak F(s',u,s)$ as the actuator that generates $u$ becomes inoperative;
\item $R: S\times \all{\coll{U}}\to \Reals$ is the reward function;
\item $\rel:S \times \all{\coll{U}}  \to [0,1]$ is the
reliability function, modeling
 that control $u$, when executed from state
$s$, results in a failure of actuator $k$ with 
probability \mbox{$1-\rel(s,u)$};
\item $\gamma \in [0,1]$ is the discount factor.
\end{itemize}
\end{definition}
\smallskip

Let $\relm{k} = \!\!\!\max\limits_{s\in S, u\in U_k}\!\!\rel(s,u)$
be the most reliable control  
for actuator $k\in\fset{m}$ and, similarly,
define the least reliable control,
$\relmin{k} = \!\!\!\min\limits_{s\in S, u\in U_k}\!\!\rel(s,u)$.
Looking across all the actuators,
now also define $\relM = \!\!\max\limits_{k\in\fset{m}}\!\relm{k} = \!\!\!\!\max\limits_{s \in S, u \in \all{\coll{U}}}\!\!\rel(s,u)$.
We will require either that $\gamma \in [0, 1)$ or $\relM\in [0, 1)$.

Note that we do not require $0 < \relM$ as in those cases the problem will involve a finite-horizon
(there are at most $|\all{\coll{U}}|$ steps).
In fact, $\relM=0$ is meaningful, being a model 
of single-shot actuators (i.e., is the case of disposable\slash one-time use gadgets---like thruster rockets or parachutes).

The evolution of \famdp $\typemdp{M} = \langle S, m, \coll{U}, T, F, R,
\coll{\alpha}\rangle$ is straightforward:
when, at time $t$, the robot is in some state $s^t$, and it 
attempts to execute $u^t = u_j \in U_k$, 
either 
\begin{itemize}
\item with probability $\rel(s^t, u_j)$: actuator $k$ does not fail and the system transitions to state $s^{t+1}$ via 
$T(s^{t+1}, u_j, s^{t}$);
\item alternatively with probability $1-\rel(s^t, u_j)$: actuator $k$ fails and the system obeys law $F(s^{t+1}, u_j, s^{t})$. 
\end{itemize}
We have assumed here that $u_j$ was admissible because actuator $k$ remained
operable at time $t$. 
Any attempt to execute a control from a previously failed actuator is undefined.

\bigskip
A \famdp can be cast into a traditional \mdp as follows:
\noindent\begin{definition}
\label{defn:monolithic}
\famdp $\typemdp{M} = \langle S, m, \coll{U}, T, F, R, \rel, \gamma\rangle$
has an associated \emph{\monomdp} $\typemdp{M}\mon  =
\langle S', U', T', R', \gamma' \rangle$:
\begin{itemize}
\item $S' = S \times \pow{\fset{m}}$;
\item $U' = U_1 \cup U_2 \cup \dots \cup U_m$;
\item \mbox{$T'((s',I'), u, (s,I)) =$}\\[5pt]
$\left\{\begin{array}{ll}
        \frac{1}{\beta}\rel(s,u) T(s',u,s), & {\small\text{if } u \in U_{k}, k \in I, I = I'},\\[6pt]
        \multirow{2}{*}{$\frac{1}{\beta}(1-\rel(s,u)) F(s',u,s),\!\!$} & {\small\text{if }  u \in U_{k}, k \in I\setminus I'},\\
                               & \quad {\small\text{\; and } I = I' \cup \{k\}},\\[6pt]
        0, & {\small\text{otherwise},}
\end{array}\right.$\\[4pt]
\item [] where it is to be understood that $U_k \in \coll{U}$,
 and multiplier $\phantom{.}\beta = 
\left\{\begin{array}{ll}
 \relM &  {\small\text{if } \gamma = 1},\\
 1 &  {\small\text{otherwise};}
\end{array}\right.$
\item $R'\big((s,I), u\big) = R(s, u)$;
\item $\gamma' = (\gamma \cdot \beta)$.
\end{itemize}
\end{definition}

In \monomdp form, the state space is augmented by a set
indicating which actuators are operable and
solving the problem
involves producing a policy, $u^t = \pi^\opt((s^t,I^t))$,
starting from state $(s^0, \fset{m})$.
The first case of $T'$ is wherein
no actuator failure occurs; the second when
actuator $k$ fails.

As there are many potential ways actuators may fail, the
state space is exponential in $m$.  
But we note that the state space's
structure is ripe for exploitation in
data-structures (to help reduce memory requirements) and algorithms.

\section{Exploiting structure}

Instead of $M\mon$ where the state-space structure is lost, 
to expose the underlying structure
we represent the \famdp as a
directed graph with associated state data:\\[-8pt]

\begin{definition}
For \famdp $\typemdp{M} = \langle S, m, \coll{U}, T, F, R, \rel, \gamma\rangle$,
an associated \emph{value function lattice} is a directed graph $G(\typemdp{M}) = (\typegr{X}, \typegr{E})$, with
\begin{list}{}{%
\setlength{\labelsep}{0pt}
\setlength{\labelwidth}{38pt}
\setlength{\topsep}{1pt}
\setlength{\itemsep}{0pt}
\setlength{\itemindent}{-10pt}
\setlength{\leftmargin}{14pt}
}%
\item $\bullet\,$ nodes $\typegr{X}\!=\!\left\{\node{\coll{X}_r}{V_r} \,|\, \coll{X}_r \subseteq \coll{U}, V_r\!:\!S \to \Reals\right\}$, every $\coll{X}_r$ being a subset of $\coll{U}$, and each $V_r$ a value function; and
\item $\bullet\,$ the set of arcs $\typegr{E}\!=\!\big\{(\node{\coll{X}_q}{V_q}, \node{\coll{X}_d}{V_d}) \,\big|\, U_i \not\in \coll{X}_d\text{ and }\allowbreak \coll{X}_q = \coll{X}_d \cup \{U_i\}\big\}$.
\end{list}
\end{definition}

\smallskip
\noindent The unique \emph{top} of $G(\typemdp{M})$ is the node $\node{\coll{U}}{V_{\coll{U}}}$,
and the unique \emph{bottom} is the node $\node{\emptyset}{V_{\emptyset}}$.
For this degenerate case, we define the value function as
\[V_{\emptyset}(s) = \tfrac{1}{1-\gamma}\cdot\min_{u \in \all{\coll{U}}}R(s,u).\]

\begin{figure}
    \centering
    \includegraphics[scale=0.6, trim=0pt 0pt 0pt 0pt, clip]{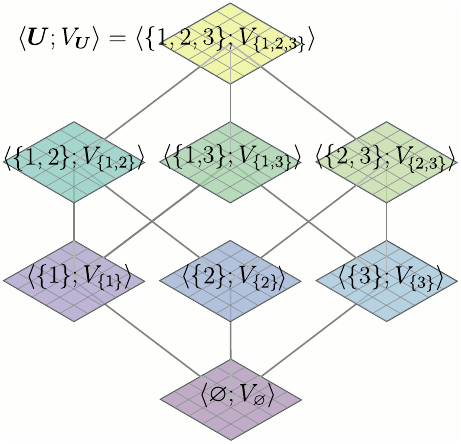}
    \caption{Value function lattice for a simple example with a set $\coll{U}$ comprising three elements. The small grids are a cartoon depiction of a $5\times 5$ state space emphasizing that each element is a value function; e.g., \mbox{$V_{\{2,3\}}:\!S \to \Reals$} assigns values to each state for the situation when $1$ has failed.\label{fig:lattice}}
    \vspace{-9pt}   
\end{figure}

\newcommand{\notationinset}{%
\begin{wrapfigure}[5]{l}{0.30\linewidth}
\vspace*{-2.7ex}
\includegraphics[scale=0.57, trim=10pt 55pt 75pt 55pt, clip]{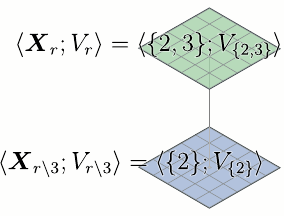}
\end{wrapfigure}%
}

For node $\node{\coll{X}_r}{V_r} \in \typegr{X}$ with $U_\ell \in \coll{X}_r$, we introduce the
notation $\coll{X}_{r\setminus U_\ell} \coloneqq \coll{X}_r \setminus \{U_\ell\}$, and immediately extend its use so that
$\node{\coll{X}_r}{V_r}$ is connected by an arc to $\node{\coll{X}_{r \setminus U_\ell}}{V_{r \setminus U_\ell}}$.
It is helpful to visualize the available actuators as being described in this directed 
graph, arranged 
 \notationinset 
in layers per set
cardinality, from $m$ at the top, to
$0$ at the bottom as depicted in \autoref{fig:lattice}. Actuator failures occur in a directed way on this graph,
moving along an arc. 
We can think of local copies of state space $S$, along with the associated 
value functions for each node. 
As the robot moves, it occupies a state in $S$ and the transition dynamics,
though moving in $S$, either stay in the node when there is no failure,
or move down to a node in the layer below.

The core idea is to perform a specialized form of value iteration on the value
functions associated at nodes because it allows us to exploit the
inter-relationships between value functions.  There are two specific forms.
Failure is monotonic: if actuator $k$ fails, the controls in $U_k$ become
unavailable ever after---the propagation of
updates during value iteration respects a partial order on all the states
induced by this monotonicity.  Secondly, a value function at one layer yields
a lower-bound for value functions above.

To apply value iteration at some particular node in $G(\typemdp{M})$, we
introduce a local operator. For node $\node{\coll{X}_r}{V_r}$,
a single Bellman update at state $s \in S$ is as follows:
\begin{align*}
V_r^\idx{i+1}(s)\!  =\!\!\! &\max_{\substack{u \in U_k\\U_k \in \coll{X}_r}} \bigg[ \gamma\!\sum_{s'\in S}\!\Big( \rel(s,u) T(s',u,s)V_r^\idx{i}(s')\;+\\[-8pt] &\;\;(1-\rel(s,u))F(s',u,s)V^\idx{i}_{r\setminus U_k}(s')\Big) + R(s,u)  \bigg].
\end{align*}

Let $\babybell$ denote the operator defined so that, given $\node{\coll{X}_r}{V_r}$, the update is
performed across the vector of states:
\begin{equation*}
V_r^\idx{i+1} = \babybell V_r^\idx{i}.
\end{equation*}
Here $\babybell$ is a \emph{local} Bellman operator in the sense that only
$V_r$ and the value functions immediately reached by
arcs departing $\node{\coll{X}_r}{V_r}$ are involved.

First, we establish that this operator converges to a fixed value function by
showing that it is a contraction.  Given two different value functions at some
node, applying $\babybell$ brings them
closer together, quantified through the \emph{max norm}:
\begin{equation}
\maxnorm{V} = \max_{s \in S} \abs{V(s)}.
\end{equation}

\begin{property}
\label{prop:contraction}
Local Bellman operator $\babybell$ is a contraction.
\end{property}
\begin{proof}
For actuator set $\coll{X}_r$, consider the two value functions
$V_r^\idx{i}$ and ${V'}_r^\idx{i}$:\\[-22pt]

{\footnotesize\par
\begin{align*}
\big\|\babybell V_r^\idx{i}  & - \babybell {V'}_r^\idx{i}\big\| \\
& \leq 
            \max_{s \in S} 
\max_{\substack{u \in U_k\\U_k \in \coll{X}_r}}\bigg|\gamma\rel(s,u)\sum_{s'\in S}T(s',u,s)\Big(V_r^\idx{i}(s') -  {V'}_r^\idx{i}(s')\Big)\bigg|\\
& \leq \gamma \max_{\substack{u \in U_k\\U_k \in \coll{X}_r}} \relm{k} \cdot
\max_{s \in S}\sum_{s'\in S}T(s',u,s)\Big|V_r^\idx{i}(s') -  {V'}_r^\idx{i}(s')\Big|\\
& \leq \gamma \max_{\substack{u \in U_k\\U_k \in \coll{X}_r}} \relm{k} \cdot
\max_{s' \in S}\Big|V_r^\idx{i}(s') -  {V'}_r^\idx{i}(s')\Big|\\
& = \gamma\cdot\relM\cdot \maxnorm{V_r^\idx{i} -  {V'}_r^\idx{i}}.
\end{align*}
}\vspace*{-12pt}

Finally, the contraction follows as factor $\gamma\cdot\relM < 1$.
\end{proof}
\smallskip

We shall use of the following (based upon \cite[p.~150]{puterman94mdp}):
\begin{lemma}
\label{lem:puter}
{\small\par $\smallmaxnorm{V^\opt - V^\idx{n}} \leq  \lambda  \smallmaxnorm{V^\idx{n} - V^\idx{n-1}}$\\
\phantom{.}\hfill where $\lambda = \dfrac{\lb}{1-\lb}$.\quad\phantom{.}}
\end{lemma}
\begin{proof}
~\\[-20pt]
{\footnotesize\par
\begin{align*}
\smallmaxnorm{V^\idx{m+n} - V^\idx{n}} 
&\leq \sum_{k=0}^{m-1}{\smallmaxnorm{V^\idx{n+k+1} - V^\idx{n+k}}}\\
&\leq \sum_{k=0}^{m-1}{\smallmaxnorm{\babybell^{k} V^\idx{n+1} - \babybell^{k} V^\idx{n}}}\\
&\leq \sum_{k=0}^{m-1}{{(\gamma\cdot\relM)^{k}}\smallmaxnorm{V^\idx{n+1} - V^\idx{n}}}\\
&\leq {\frac{1-\lb^m}{1-\lb}}\smallmaxnorm{V^\idx{n+1} - V^\idx{n}}.
\end{align*}
}

Hence, the sequence $V^\idx{n}, V^\idx{n+1}, V^\idx{n+2}, V^\idx{n+3}, \cdots$ is Cauchy, and we can take the limit
as $m \to \infty$. 
We obtain the final result unrolling once using Property~\ref{prop:contraction}.
\end{proof}

\bigskip

Next, we examine the relationship between approximation errors 
within the value function lattice.
Consider
an arbitrary node $\node{\coll{X}_r}{V_r}$, where
$\coll{X}_r = \{U_1, U_2, U_3, \dots, U_{N_r}\}$. 
It accordingly has arcs connecting to
nodes $\node{\coll{X}_{r \setminus U_1}}{V_{r \setminus U_1}},$
$\node{\coll{X}_{r \setminus U_2}}{V_{r \setminus U_2}},$ $\dots,$
$\node{\coll{X}_{r \setminus U_{N_r}}}{V_{r \setminus U_{N_r}}}$.  
Denote the global
fixed point value functions as $V^\opt_r$ for the node itself, and
$V^\opt_{r\setminus U_j}$ with $j\in\fset{N_r}$ for those the layer below.
(One can imagine this fixed point either obtained from the monolithic 
problem as in \autoref{defn:monolithic}, or 
\gobble{instead }
as the fixed point
under local Bellman operator, computed upward from the lattice bottom.)

We analyze an update to $V_r^\idx{i}$, assuming that, for each $j \in \fset{N_r}$,
we have value functions $V_{r \setminus U_j}$ with error at most~$\varepsilon_j$.
Then, the error in some Bellman iterate is bounded:

{\small
\begin{align*}
\big\|&\babybell V_r^\idx{i} - {V}_r^\opt\big\| =\big\|V_r^\idx{i+1} - V_r^\opt\big\| \\
& =            \max_{s \in S} \abs{ V_r^\idx{i+1}(s,u) - V_r^\opt(s,u) }\\
& =
            \gamma \max_{s \in S} 
\max_{\substack{u \in U_k\\U_k \in \coll{X}_r}}\!\bigg|\sum_{s'\in S}\rel(s,u) T(s',u,s)\Big(V_r^\idx{i}(s') -  V_r^\opt(s')\Big) +  \\[-4pt]
&\quad\qquad
\sum_{s'\in S}(1-\rel(s,u)) F(s',u,s)\Big(V_{r\setminus U_k}^\idx{i}(s') -  V_{r\setminus U_k}^\opt(s')\Big) \bigg|\\
& \leq 
            \gamma \max_{\substack{u \in U_k\\U_k \in \coll{X}_r}}\!\Bigg(
\max_{s \in S} \bigg|\sum_{s'\in S}\rel(s,u) T(s',u,s)\Big(V_r^\idx{i}(s') -  V_r^\opt(s')\Big)\bigg| +  \\[-4pt]
&\qquad
\max_{s \in S} \bigg|\sum_{s'\in S}(1-\rel(s,u)) F(s',u,s)\Big(V_{r\setminus U_k}^\idx{i}(s') -  V_{r\setminus U_k}^\opt(s')\Big) \bigg| \Bigg)\\
& \leq
            \gamma \max_{\substack{u \in U_k\\U_k \in \coll{X}_r}}\!\bigg(
\relm{k} \cdot \max_{s' \in S} \Big|\big(V_r^\idx{i}(s') -  V_r^\opt(s')\big)\Big| +  \\[-4pt]
&\quad\quad\qquad\qquad
(1-\relmin{k}) \cdot \max_{s' \in S} \Big|\big(V_{r\setminus U_k}^\idx{i}(s') -  V_{r\setminus U_k}^\opt(s')\big)\Big| \bigg)  \\[-4pt]
& = 
            \gamma \max_{\substack{u \in U_k\\U_k \in \coll{X}_r}}\!\bigg(
\relm{k} \cdot \maxnorm{V_r^\idx{i} -  V_r^\opt} + 
(1-\relmin{k}) \cdot \maxnorm{V_{r\setminus U_k}^\idx{i} -  V_{r \setminus U_k}^\opt}\bigg)\\
& \leq \gamma \max_{k \in \fset{{N_r}}}\!\bigg(
\relm{k} \cdot \maxnorm{V_r^\idx{i} -  V_r^\opt} + 
(1-\relmin{k}) \cdot \varepsilon_k\bigg).
        \end{align*}}
To bound $\maxnorm{\babybell V_r^\idx{i} - V_r^\opt}$ by $\varepsilon_r$, it is enough
to have 

{\small
\begin{align*}
\max_{k \in \fset{N_r}}\!\bigg( \relm{k} \cdot \maxnorm{V_r^\idx{i} -  V_r^\opt} + (1-\relmin{k}) \cdot \varepsilon_k\bigg) & \leq \frac{\varepsilon_r}{\gamma},
\end{align*}}
but we can rewrite it, using Lemma~\ref{lem:puter}, 

{\small
\begin{align*}
\max_{k \in \fset{N_r}}\!\bigg( \relm{k} \cdot \maxnorm{V_r^\idx{i} -  V_r^\opt} + (1-\relmin{k}) \cdot \varepsilon_k\bigg) &\leq \\
\max_{k \in \fset{N_r}}\!\bigg( \relm{k} \cdot \lambda \maxnorm{V_r^\idx{i} - V_r^\idx{i-1}} + (1-\relmin{k}) \cdot \varepsilon_k\bigg)  &\leq \frac{\varepsilon_r}{\gamma},
\end{align*}}
and, as a bound on the maximum, we can simply require:

{\small
\begin{align*}
\maxnorm{V_r^\idx{i} - V_r^\idx{i-1}}   \leq \min_{k\in\fset{N_r}}\Bigg[\frac{1}{\lambda}\bigg(\frac{\varepsilon_r}{\gamma\relm{k}} - \frac{(1-\relmin{k})}{\relm{k}} \cdot \varepsilon_k\bigg)\Bigg].
\end{align*}}
Further, when each node $\node{\coll{X}_r}{V_r} \in \typegr{X}$ is subject to the same quality constraint $\varepsilon_r$ we have:

{\small
\begin{align*}
\maxnorm{V_r^\idx{i} - V_r^\idx{i-1}}   \leq \min_{k\in\fset{N_r}}\Bigg[\frac{\varepsilon_r}{\lambda\relm{k}}\bigg(\frac{1}{\gamma} + \relmin{k} - 1\bigg)\Bigg].
\end{align*}}

\section{Solving \famdp Problems}
The simplest approach to solving \famdp problems is to construct a monolithic \famdp which can be treated by a general \mdp solver. 
We define a planner utilizing this approach as a \monoplanner.
By introducing a value function lattice we can reduce the time required to converge to an optimal policy.
This requires a specialized solver which can operate on the value function lattice directly, essentially treating the \famdp as a lattice of smaller \mdps. 
A planner which utilizes the value function lattice is called a \latticeplanner.

Since the value function lattice is a directed acyclic graph (DAG), we can use a method similar to Topological Value Iteration~\cite{dai2011topovi} to solve each node to a given quality requirement $\varepsilon_{\rm{desired}}$. This is done from the bottom upward, where each node is solved to the quality constraint once those below it have already been solved. The binary failure property of actuators guarantees this approach since every node is causally dependent on only the nodes below it. 
To solve an \famdp subject to the quality requirement $\varepsilon_{\rm{desired}}$ using a \latticeplanner we first compute the maximum Bellman error $y_r$ required to satisfy $\varepsilon_{\rm{desired}}$ for each node $\node{\coll{X}_{r}}{V_{r}}$. 
If each node is held to the same quality constraint we have

{\small
 \vspace*{-2ex}
\begin{align*}
    y_r = \min_{j\in\fset{N_r}}\Bigg[\frac{\varepsilon_{\rm{desired}}}{\lambda\relm{j}}\bigg(\frac{1}{\gamma} + \relmin{j} - 1\bigg)\Bigg].
\end{align*}}
\noindent We then perform value iteration on each node in the lattice in a bottom-up fashion. 
Note that in contrast to Topological Value Iteration~\cite{dai2011topovi}, we have the useful properties that the meta-state orderings enable, but without having to pay the cost of actually computing strongly connected components.

For an additional speedup we can leverage the similarity between actuator combinations to propagate values up the lattice. The monotonicity of failure ensures that value functions at every node in the lattice serves as a lower bound for the nodes above them.
As a result, we can \emph{hot-start} the value function at each node by taking the maximum value 
at each state in the nodes below.
A \latticeplanner which employs \hotstarting is called a \latticeplannerhotstart.

\section{Experiments}
To measure the performance of each planner, we track the planning time required via the total number of reads and writes to the value functions. 
This metric is a meaningful approximation of a planner's running time because it is implementation agnostic and removes noisy runtime clock measurements.
Our implementation (in Python) limits the scale of problems we can evaluate in a practice---future work might optimize the implementation to investigate \famdps with larger state spaces and more actuators.
\begin{figure}
    \centering
    \includegraphics[width=\linewidth, trim=5pt 5pt 5pt 5pt, clip]{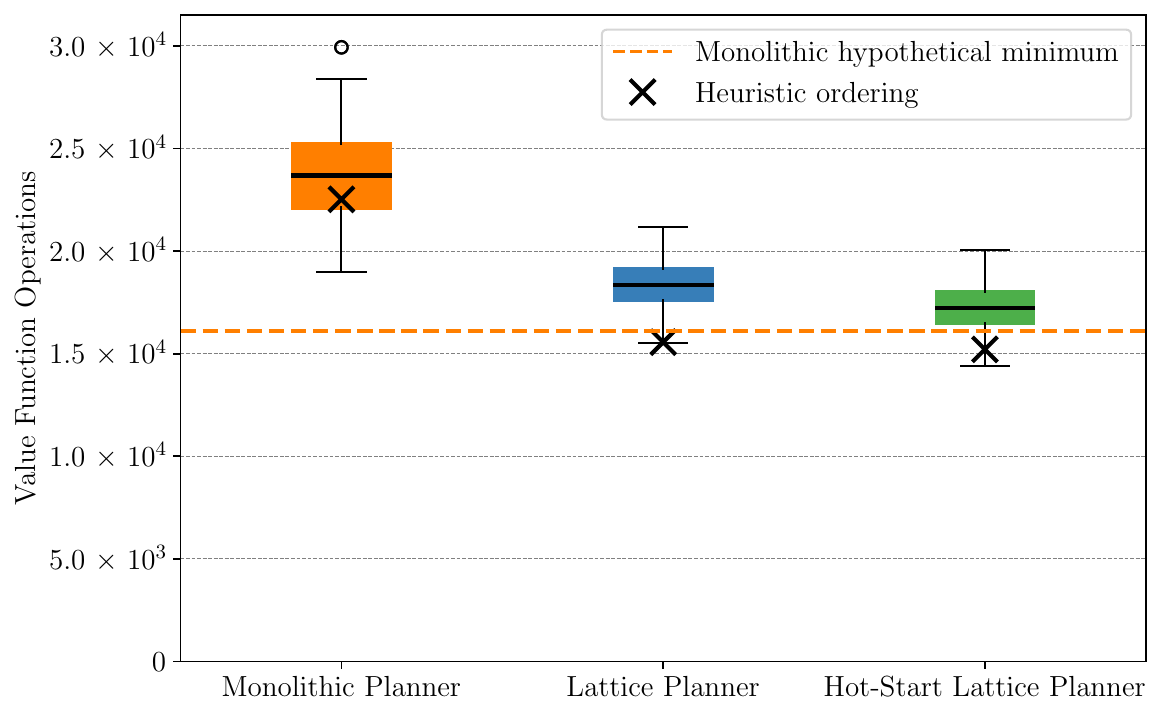}
    \caption{Value Function Operations \emph{vs} State Backup Ordering: ($\gamma=0.99$, $\varepsilon_{\tgt}=0.001$) Comparison of state backup orderings during asynchronous value iteration. 
    Ordering distributions are represented by the box and whisker plots.}  
    \label{fig:mono_famdp_comp}
\end{figure}
\subsection{Methodology}
To explore the properties of \famdp problems we will use a 2D \gridworld which provides a simple yet scalable environment to test on.
Each cell in the grid corresponds to a state and is associated with a \textit{terrain} property.
It is convenient to define the transition dynamics $T$ and $F$, the reward $R$, and the reliability $\rel$ functions with respect to this property.
The terrain dictates the controls available in a given state, with only a subset of the actuators working in each terrain. 
If an actuator is available in a state, it has a set of controls for moving to directly adjacent cells.
The transition dynamics are dictated by a single \emph{precision} value $p \in [0,1]$, which is the probability that a given control will transition to the desired next state, assuming no failure occurs. 
For example if the agent is in state $s \in S$, control $u \in U_k$ is \emph{}{up} and the state above $s$ is $s' \in S$, then $T(s',u,s) = a$. 
If state $s$ has $n$ states adjacent to it, then for every remaining state $s_{i}, i \in \fset{n-1}$, 

{\small
\vspace*{-4ex}
\begin{align*}
     T(s_{i},u,s) = \frac{(1-a)}{(n-1)}.
\end{align*}
\vspace*{-1ex}
}

The transition dynamics for when an actuator fails $F$ are defined similarly, but usually with a reduced precision value.

The problem instance includes a single goal state $s_{\goal}$ where $T(s_{\goal},u,s_{\goal}) = 1, \rel(s_{\goal},u) = 1$, and $R(s_{\goal}, u)$ is a positive reward $R_{\goal}$ for all controls $u \in \all{\coll{U}}$.
The value functions are initialized such that 
\begin{align*}
    V^{(0)}(s) = \min_{s\in S, u\in \coll{U}}\frac{R(s,u)}{1 - \gamma}\text{\;  for all }s \in S\setminus\{s_{\goal}\}, 
\end{align*}
analogous to being stranded for eternity. 
At $s_{\goal}$ we have
\begin{align*}
    V^{(0)}(s_{\goal}) = \frac{R_{\goal}}{1 - \gamma},
\end{align*}
analogous to using any control in the goal state continuously.

\subsection{Asynchronous Value Iteration State Backup Ordering}

For asynchronous value iteration, the sequence in which states undergo value function backup plays a pivotal role in determining the computation required for convergence. 
This study investigates the ramifications of state ordering on algorithmic performance, so that we might distinguish between different performance factors. 
We also lay the groundwork for an evaluation of ordering heuristics.

\autoref{fig:mono_famdp_comp} illustrates the impact ordering plays on performance across planners. 
For each of the planners, we randomly sample 500 unique backup orderings from a uniform distribution and perform value iteration. 
Note that the definition of a random ordering is different for a \monoplanner and a \latticeplanner.
Since the monolithic \famdp does not preserve the partial order imposed by the actuators over the full state space, a random ordering for a \monoplanner is a random order across all $|S|2^{m}$ states.
For a \latticeplanner, a random ordering is defined as a random order over the local copy of $|S|$ states within each node. 
As a result, the \monoplanner has the largest performance range. 

To better analyze the backup order implications on the \monoplanner, we propose a hypothetical minimum number of operations it requires. 
We define this minimum by utilizing an \emph{oracle} to provide the converged value function $V^*$.
The hypothetical best order (via an oracle) is then defined as a partial order over the states subject to the values provided by the oracle, with the highest valued states being visited first.
We then can define the hypothetical minimum number of operations as the number of operations required for a \monoplanner to converge to $\varepsilon_{\tgt}$ using the hypothetical best order.
In practice, we obtain the oracle ordering by running value iteration until convergence to some $\varepsilon_{\tgt}$ and use that value function as a reasonable approximation of $V^*$. 
We then create our hypothetical best order informed by this value function and evaluate the total number of operations required for the \monoplanner to converge with this ordering.
The orange line in \autoref{fig:mono_famdp_comp} represents the hypothetical minimum operations. 
Note that the \latticeplanner with the Manhattan distance heuristic ordering outperforms the hypothetical best-case ordering for the \monoplanner. 
This indicates that the \latticeplanner{s} benefit not only from the partial ordering imposed by the lattice structure, but also from the decomposition into smaller \mdps.
Despite the small size of the state space and the limit actuator set, the \latticeplanner{s} unquestionably outperform the \monoplanner. 

To effectively scale these planners for more complex \mdps, it's crucial to adopt an ordering heuristic tailored to the  problem domain. 
Our evaluation of the Manhattan distance heuristic demonstrates its
superiority over most random orderings on each of the planners, making it a
good choice as 
the default ordering heuristic for subsequent tests.

\begin{figure}
    \centering
    \includegraphics[width=\linewidth, trim=10pt 10pt 10pt 10pt, clip]{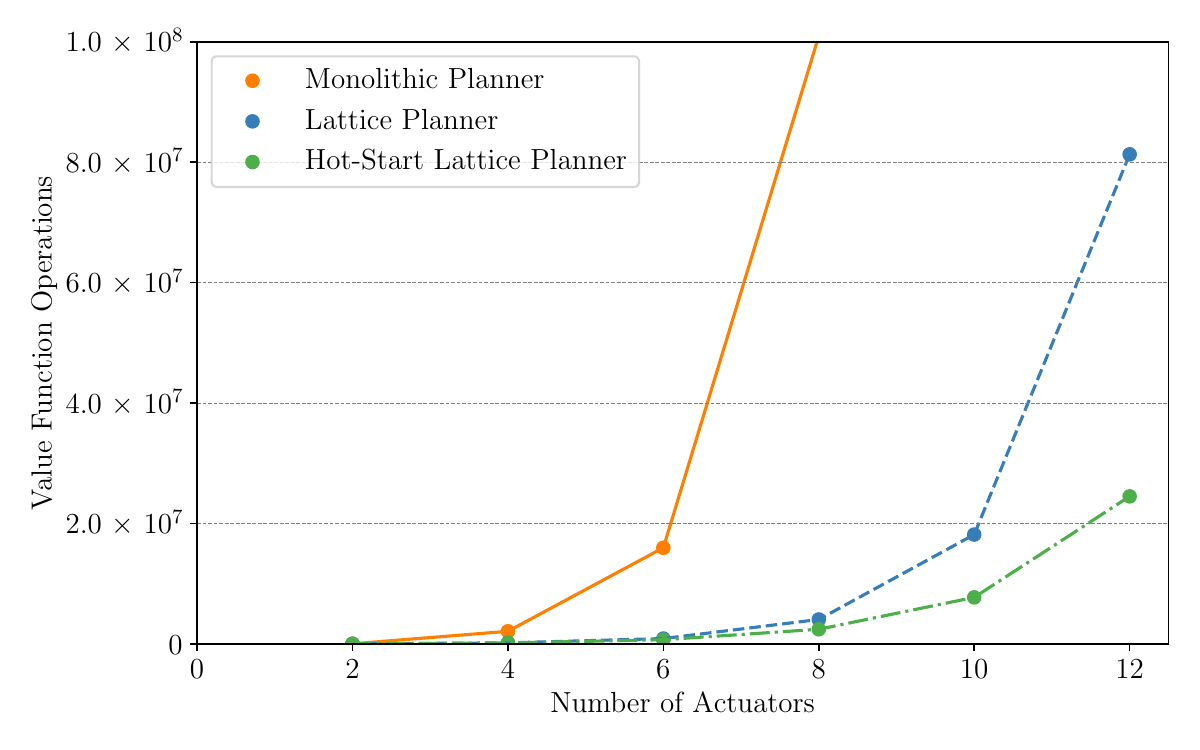}
    \caption{Value Function Operations \emph{vs} Number of Actuators: Scaling the number of actuators from 2 to 12 using the \gridworld from \autoref{fig:famdpPlannerTop}.}
    \label{fig:actuator-scaling}
\end{figure}

\subsection{Scaling}
To characterize the performance of each planner as the size of the lattice grows, we increase the number of actuators $m$ in the baseline $6\times6$ \gridworld from 2 up to 12. 
For consistency, we create an additional copy of each actuator for each increasing scale. 
Thus an \famdp with 12 actuators contains 6 copies of the 2 original actuators.
    
\autoref{fig:actuator-scaling} demonstrates why the exponential nature of \famdp problems necessitates planner efficiency. 
The \monoplanner rapidly becomes infeasible beyond 6 actuators, even with a small grid containing 36 states. 
While all planners have the same asymptotic complexity, the \latticeplanner{s} are able to scale to significantly more actuators before becoming infeasible. 
Here the value of \hotstarting is made apparent as the number of actuators increases.
Despite the small grid, \hotstarting enables planning on more complicated robots as the \latticeplannerhotstart with 12 actuators performs a similar number of operations as the \latticeplanner with only 10.

\section{Conclusion}
When suitably interpreted, discount factors within \mdps express an all-or-nothing model of breakdown.  
The present paper refines that model to enable treatment of incremental, component-wise failure for the case of actuators that cease proper functioning with a likelihood proportional to their use. 
This allows one to model a robot which degrades across time and to plan operations strategically by choosing actions that anticipate a future in which the hardware is less capable.  
In essence, our approach augments the notion of state with characteristics of robot's operational hardware.
While this means that standard \mdp solvers can treat such problems, being general solvers, they are necessarily oblivious to the specific structure which arises when characterizing malfunction.
We demonstrate that the lattice structure produces a more efficient solution through 
3 separate means: 
(1) the partial order induced by the subset structure arising from set inclusion of actuators; (2) decomposition into a set of smaller \mdps; (3) the similarity of value functions 
for sets of actuators that have minor differences.
Our empirical results show that the advantage afforded by these 
three 
ideas increases with the addition of actuators, i.e., as robots become more complex.

\bibliographystyle{IEEEtranS}
\bibliography{references}

\begin{thebibliography}{10}
\providecommand{\url}[1]{#1}
\csname url@samestyle\endcsname
\providecommand{\newblock}{\relax}
\providecommand{\bibinfo}[2]{#2}
\providecommand{\BIBentrySTDinterwordspacing}{\spaceskip=0pt\relax}
\providecommand{\BIBentryALTinterwordstretchfactor}{4}
\providecommand{\BIBentryALTinterwordspacing}{\spaceskip=\fontdimen2\font plus
\BIBentryALTinterwordstretchfactor\fontdimen3\font minus \fontdimen4\font\relax}
\providecommand{\BIBforeignlanguage}[2]{{%
\expandafter\ifx\csname l@#1\endcsname\relax
\typeout{** WARNING: IEEEtranS.bst: No hyphenation pattern has been}%
\typeout{** loaded for the language `#1'. Using the pattern for}%
\typeout{** the default language instead.}%
\else
\language=\csname l@#1\endcsname
\fi
#2}}
\providecommand{\BIBdecl}{\relax}
\BIBdecl

\bibitem{ahmadzadeh2015modular}
H.~Ahmadzadeh and E.~Masehian, ``Modular robotic systems: Methods and algorithms for abstraction, planning, control, and synchronization,'' \emph{Artificial Intelligence}, vol. 223, pp. 27--64, 2015.

\bibitem{amin2019}
A.~A. Amin and K.~M. Hasan, ``A review of fault tolerant control systems: Advancements and applications,'' \emph{Measurement}, vol. 143, pp. 58--68, 2019.

\bibitem{bertsekas19reinforcement}
D.~P. Bertsekas, \emph{{Reinforcement Learning and Optimal Control}}.\hskip 1em plus 0.5em minus 0.4em\relax Belmont, M.A., U.S.A: Athena Scientific, 2019.

\bibitem{crestani15enhancing}
D.~Crestani, K.~Godary-Dejean, and L.~Lapierre, ``{Enhancing fault tolerance of autonomous mobile robots},'' \emph{Robotics and Autonomous Systems}, vol.~68, pp. 140--155, 2015.

\bibitem{dai2011topovi}
P.~Dai, Mausam, D.~S. Weld, and J.~Goldsmith, ``Topological value iteration algorithms,'' \emph{Journal of Artificial Intelligence Research}, vol.~42, pp. 181--209, 2011.

\bibitem{donald88geometric}
B.~R. Donald, ``{A geometric approach to error detection and recovery for robot motion planning with uncertainty},'' \emph{Artificial Intelligence}, vol.~37, no.~1, pp. 223--271, 1988.

\bibitem{gautam22}
A.~Gautam, T.~Whiting, X.~Cao, M.~A. Goodrich, , and J.~W. Crandall, ``{A Method for Designing Autonomous Robots that Know Their Limits},'' in \emph{Proc. of IEEE International Conference on Robotics and Automation (ICRA)}, 2022, (to appear).

\bibitem{hsiao2012hierarchical}
T.~Hsiao and M.-C. Weng, ``A hierarchical multiple-model approach for detection and isolation of robotic actuator faults,'' \emph{Robotics and Autonomous Systems}, vol.~60, no.~2, pp. 154--166, 2012.

\bibitem{khalastchi2018fault}
E.~Khalastchi and M.~Kalech, ``On fault detection and diagnosis in robotic systems,'' \emph{ACM Computing Surveys (CSUR)}, vol.~51, no.~1, pp. 1--24, 2018.

\bibitem{kumar2018reliability}
V.~Kumar, L.~Singh, and A.~K. Tripathi, ``Reliability analysis of safety-critical and control systems: a state-of-the-art review,'' \emph{IET Software}, vol.~12, no.~1, pp. 1--18, 2018.

\bibitem{lavalle06planning}
S.~M. LaValle, \emph{{Planning Algorithms}}.\hskip 1em plus 0.5em minus 0.4em\relax Cambridge, U.K.: Cambridge University Press, 2006, available at http://planning.cs.uiuc.edu/.

\bibitem{lu2023joint}
Y.~Lu, H.~R. Karimi, and N.~Zhang, ``Joint actuator fault estimation and localization under round-robin protocol for wheeled mobile robot,'' \emph{Transactions of the Institute of Measurement and Control}, 2023.

\bibitem{lyness2019rwt}
H.~Lyness and D.~Apostolopoulos, ``An evaluation of reconfigurable wheel-track testbeds for military vehicles,'' in \emph{Proc. of 15th European-African Regional Conference of the International Society for Terrain Vehicle Systems Conference (ISTVS '19)}, September 2019.

\bibitem{olsen22robust}
T.~Olsen, N.~Stiffler, and J.~M. O'Kane, ``Robust-by-design plans for multi-robot pursuit-evasion,'' in \emph{Proc. IEEE International Conference on Robotics and Automation}, 2022.

\bibitem{prorok2021beyond}
A.~Prorok, M.~Malencia, L.~Carlone, G.~S. Sukhatme, B.~M. Sadler, and V.~Kumar, ``Beyond robustness: A taxonomy of approaches towards resilient multi-robot systems,'' \emph{arXiv preprint arXiv:2109.12343}, 2021.

\bibitem{puterman94mdp}
M.~L. Puterman, \emph{Markov Decision Processes: Discrete Stochastic Dynamic Programming}.\hskip 1em plus 0.5em minus 0.4em\relax Hoboken, N.J., U.S.A: Wiley-Interscience, 1994.

\bibitem{russell09artificial}
S.~J. Russell and P.~Norvig, \emph{{Artificial Intelligence: A Modern Approach}}, 3rd~ed.\hskip 1em plus 0.5em minus 0.4em\relax Upper Saddle River, NJ, U.S.A.: Prentice-Hall, Inc., 2009.

\bibitem{srinivas77thesis}
S.~Srinivas, ``{Error Recovery in Robot Systems},'' {PhD} Dissertation, California Institute of Technology, 1977.

\bibitem{stancliff2009taskallocation}
S.~B. Stancliff, J.~Dolan, and A.~Trebi-Ollennu, ``Planning to fail—reliability needs to be considered a priori in multirobot task allocation,'' in \emph{2009 IEEE International Conference on Systems, Man and Cybernetics}.\hskip 1em plus 0.5em minus 0.4em\relax IEEE, 2009, pp. 2362--2367.

\bibitem{sutton18reinforcement}
R.~S. Sutton and A.~G. Barto, \emph{{Reinforcement Learning: An Introduction}}, 2nd~ed.\hskip 1em plus 0.5em minus 0.4em\relax Cambridge, M.A., U.S.A.: MIT Press, 2018.

\bibitem{ure2013uav}
N.~K. Ure, G.~Chowdhary, J.~P. How, M.~A. Vavrina, and J.~Vian, ``Health aware planning under uncertainty for uav missions with heterogeneous teams,'' in \emph{2013 European Control Conference (ECC)}.\hskip 1em plus 0.5em minus 0.4em\relax IEEE, 2013, pp. 3312--3319.

\bibitem{verma2004real}
V.~Verma, G.~Gordon, R.~Simmons, and S.~Thrun, ``{Real-Time Fault Diagnosis: Tractable Particle Filters for Robot Fault Diagnosis},'' \emph{IEEE Robotics \& Automation Magazine}, vol.~11, no.~2, pp. 56--66, 2004.

\bibitem{visinsky94robotic}
M.~L. Visinsky, J.~R. Cavallaro, and I.~D. Walker, ``{Robotic fault detection and fault tolerance: A survey},'' \emph{Reliability Engineering \& System Safety}, vol.~46, no.~2, pp. 139--158, 1994.

\bibitem{walker2022robot}
I.~D. Walker, ``Robot fault detection: At the analog-digital boundary,'' in \emph{2022 IEEE Annual Reliability and Maintainability Symposium (RAMS)}, 2022, pp. 1--6.

\bibitem{xiao2020robust}
B.~Xiao, L.~Cao, S.~Xu, and L.~Liu, ``{Robust Tracking Control of Robot Manipulators With Actuator Faults and Joint Velocity Measurement Uncertainty},'' \emph{IEEE/ASME Transactions on Mechatronics}, vol.~25, no.~3, pp. 1354--1365, 2020.

\bibitem{yoo2012actuator}
S.~Yoo, ``Actuator fault detection and adaptive accommodation control of flexible-joint robots,'' \emph{IET control theory \& applications}, vol.~6, no.~10, pp. 1497--1507, 2012.

\bibitem{zhang2017resilient}
T.~Zhang, W.~Zhang, and M.~M. Gupta, ``{Resilient Robots: Concept, Review, and Future Directions},'' \emph{Robotics---Special Issue: Robust and Resilient Robots}, vol.~6, no.~4, 2017.

\end{thebibliography}
\end{document}